%% file: main.tex
\documentclass[11pt]{article}

\usepackage[preprint]{acl}

\usepackage{times}
\usepackage{latexsym}

\usepackage[T1]{fontenc}

\usepackage[utf8]{inputenc}

\usepackage{microtype}

\usepackage{inconsolata}
\usepackage{graphicx}
\usepackage{amsmath}
\usepackage{amsthm}
\usepackage{amssymb}
\usepackage{mathtools}
\usepackage{booktabs}
\usepackage{multirow}
\usepackage{graphicx}
\usepackage{colortbl}

\usepackage{algorithm}
\usepackage{algorithmic}
\usepackage{bm}
\usepackage{xcolor}
\usepackage{letterspace}

\usepackage{enumitem}
\usepackage{tcolorbox}

\definecolor{algblue}{RGB}{0, 80, 180}  
\definecolor{commentgray}{RGB}{100, 100, 100}

\definecolor{ourcolor}{RGB}{230, 240, 245}

\newcommand{\m}{\textsc{T-STAR}}

\newtheorem{lemma}{Lemma}

\title{Reason in Chains, Learn in Trees: Self-Rectification and Grafting for Multi-turn Agent Policy Optimization}

\author{Yu Li \quad Sizhe Tang \quad Tian Lan\textsuperscript{*} \\
  Department of Electrical and Computer Engineering,
  George Washington University \\
  \texttt{\{yul,tlan\}@gwu.edu} \\
  \textsuperscript{*}Corresponding author
}

\begin{document}
\maketitle

\input{section/0_abstract}
\input{section/1_introduction}

\input{section/2_rw}

\input{section/3_method}

\input{section/4_result}
\input{section/5_conclusion}

\newpage
\section*{Limitations}
The functional equivalence criterion based on KL divergence relies on Monte Carlo approximation, which may introduce noise for states with high-entropy action distributions
but can also be addressed by adaptive sampling strategies in future work. 
Additionally, thought grafting requires the agent to generate meaningful corrections conditioned on contrasting branches, which depends on the base model's ability to identify and articulate reasoning differences. Finally, our evaluation focuses on text-based environments with discrete actions; extending \m~to continuous action spaces or multimodal observations would be interesting topics for future exploration.

\bibliography{ref}

\input{section/x_appendix}

\end{document}

%% file: section/0_abstract.tex
\begin{abstract}
Reinforcement learning for Large Language Model agents is often hindered by sparse rewards in multi-step reasoning tasks. Existing approaches like Group Relative Policy Optimization treat sampled trajectories as independent chains, assigning uniform credit to all steps in each chain and ignoring the existence of critical steps that may disproportionally impact reasoning outcome. 
In this paper, we propose \m~(Tree-structured Self-Taught Agent Rectification), a framework that recovers the latent correlated reward structure across seemingly independent trajectories. Specifically, we consolidate trajectories into a unified Cognitive Tree by identifying and merging functionally similar steps/nodes. It enables an Introspective Valuation mechanism that back-propagates trajectory-level rewards through the tree to obtain a new notion of variance-reduced relative advantage at step-level. Using the Cognitive Tree, we also develop In-Context Thought Grafting to synthesize corrective reasoning by contrasting successful and failed branches at critical divergence points/steps. Our proposed Surgical Policy Optimization then capitalizes on the rich policy gradient information concentrated at these critical points/steps through a Bradley-Terry type of surgical loss. 
Extensive experiments across embodied, interactive, reasoning, and planning benchmarks demonstrate that \m~achieves consistent improvements over strong baselines, with gains most pronounced on tasks requiring extended reasoning chains.
\end{abstract}

%% file: section/1_introduction.tex
\section{Introduction}

Reinforcement learning has emerged as a powerful post-training paradigm for Large Language Models, enabling capabilities ranging from planning to complex mathematical reasoning~\citep{wang2025reinforcement,kumar2025llm,park2025maporl,wang2025think}. 
Recent advances extend this paradigm to LLM agents 
with multi-turn reasoning cycles, tackling tasks such as web navigation, embodied control, and tool-augmented question answering~\citep{zeng2025reinforcing,wang2025ragen,chen2025context}. 
These settings typically adopt the ReAct framework~\citep{yao2022react}, where agents interleave reasoning thoughts with executable actions across multiple steps, with trajectories (or reasoning chains) often spanning dozens of decisions~\citep{shinn2023reflexion,yao2023tree,bai2023qwen}. 
To optimize such agents, methods like Proximal Policy Optimization (PPO) require training separate value networks to estimate advantages, adding computational overhead and potential instability in long-horizon settings~\citep{schulman2017proximal}. 
Group Relative Policy Optimization (GRPO) addresses this by estimating a notion of relative advantage through comparison within sampled trajectory groups, eliminating the need for value networks while maintaining competitive performance on multi-turn reasoning agents~\citep{guo2025deepseek}. 
Therefore extracting rich policy gradient information from finite rollouts and sparse rewards is a fundamental challenge~\citep{zheng2018learning,chiappa2023latent,wang2026perm}.   

However, existing approaches often treat sampled trajectories as independent chains and thus are oblivious to the latent correlated reward structure across them. 
We note two key limitations arising from this. First, since sparse rewards are received only at trajectory completion, all steps within each trajectory receive the same credit (despite their functional disparity), while common steps shared by multiple trajectories may receive inconsistent credit (despite their functional equivalence). Existing coarse-grained credit assignment fails to distinguish critical decision points from routine execution steps, making it difficult for agents to learn which specific reasoning choices led to success or failure~\citep{bai2024digirl,guo2025segment,zeng2025reinforcing,zhang2026learning}. 
Second, rich policy-gradient information are likely concentrated at critical decision points/steps that have a disproportional impact on reasoning outcome which may lead to suboptimal policy update~\citep{kumar2023policy,sun2023offline}.

\begin{figure}
    \centering
    \includegraphics[width=\columnwidth]{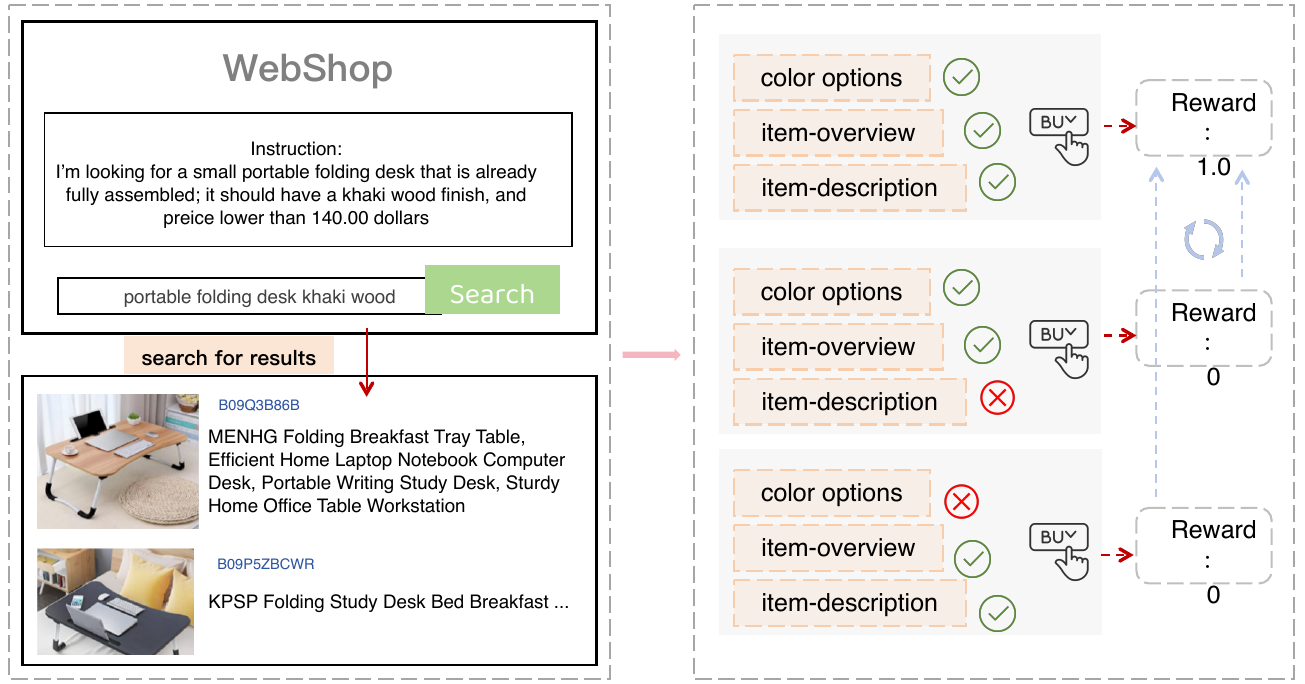}
    \caption{An illustration using WebShop Operation. Three trajectories share identical prefixes but diverge at the attribute verification step, yielding different outcomes. Existing methods with trajectory-level rewards assign inconsistent credit for the prefixes in each trajectory, while critical decision steps fail to be distinguished and utilized for policy optimization.
    }
    \label{fig:motivation}
\end{figure}

To address these limitations, we propose a novel framework, Tree-structured Self-Taught Agent Rectification~(\m), that recovers the latent reward structure across seemingly independent rollouts for improved policy optimization. 
By consolidating trajectories into a unified Cognitive Tree by merging functionally similar steps/nodes, \m~enables an Introspective Valuation mechanism that back-propagates trajectory-level rewards through the tree and constructs a variance-reduced relative advantage at each step/node. 
To capture critical decision points, \m~develops In-Context Thought Grafting joining reasoning branches sharing similar prefixes but belonging to different trajectories. 
It allows \m~to contrast successful and failed branches and enables a Surgical Policy Optimization to capitalize on the rich policy gradient information concentrated at these critical divergence steps through a Bradley-Terry type of surgical loss function. 

Our proposed \m~works without requiring additional rollouts or reward models and can be readily integrated into state-of-the-art methods including GRPO, DAPO, and GiGPO~\citep{guo2025deepseek,yu2025dapo,feng2025group}.
When multiple rollouts attempt similar tasks, they frequently traverse functionally similar intermediate states before diverging at critical junctures~\citep{zhang2025reasoning,li2025inspo,li2026right}.
A failed trajectory may contain a perfectly valid reasoning prefix that, with a single corrected decision at the divergence point, would have succeeded~\citep{huang2025tree}. 
Recognizing this latent structure could both reduce gradient variance through aggregation and precisely localize where targeted corrections are needed~\citep{yang2025treerpo}. Figure~\ref{fig:motivation} illustrates three diverging trajectories in WebShop, consolidated into a Cognitive Tree.
Not only do prefixes receive consistent credit via a variance-reduced relative advantage (by reward backpropagation in the tree), critical divergence point (at the attribute verification step) is identified and is leveraged to produce a new policy gradient component from Bradley-Terry type of surgical loss.

Comprehensive experiments across embodied, interactive, reasoning, and planning tasks demonstrate consistent improvements over GRPO and variants, with gains most pronounced on tasks requiring extended reasoning chains where trajectory overlap is frequent.
Our contributions include: (1)a Cognitive Tree construction enabling variance-reduced advantage estimation through trajectory consolidation and reward back-propagation; (2)a self-taught rectification mechanism that synthesizes step-level supervision at divergence points/steps via thought grafting; and (3) surgical policy optimization targeting critical decision steps.

%% file: section/2_rw.tex
\begin{figure*}[!t]
    \centering
    \includegraphics[width=1.0\linewidth]{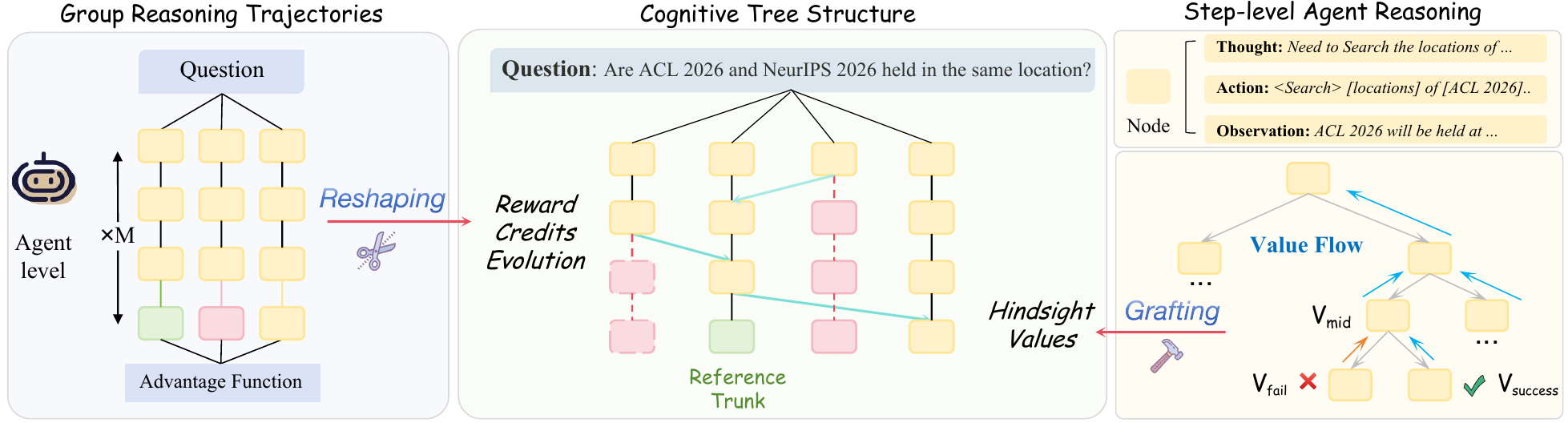}
    \caption{Overview of \m. Given $M$ sampled trajectories per task, \m~consolidates them into a Cognitive Tree by merging functionally equivalent nodes, then computes structural values via Bellman backup that propagate downstream success rates to intermediate nodes. This enables trajectory stitching where successful reasoning steps receive appropriate credit even within failed rollouts. At divergence points where children exhibit large value differences, the agent performs self-rectification by generating corrective reasoning that transfers successful logic to failed contexts, producing preference pairs for surgical policy optimization.}
    \label{fig:overview}
\end{figure*}

\section{Related Work}

\textbf{RL for LLM Agents.}
Reinforcement learning has become a key paradigm for training LLM agents \cite{zhang2025landscape,zhang2026text,zhang-etal-2025-ga}. 
PPO~\citep{schulman2017proximal} requires value networks that are computationally expensive for long-context agent tasks. 
Group-based methods like GRPO~\citep{guo2025deepseek} eliminate value networks through in-group advantage estimation, enabling more efficient training. 
Subsequent variants such as DAPO~\citep{yu2025dapo} and GiGPO~\citep{feng2025group} further improve training stability. 
These methods typically treat sampled trajectories as independent chains and suffers from the two limitations mentioned~\citep{zhang2026logicalphasetransitionsunderstanding}.

\paragraph{Self-Improvement in Reasoning.}
Sparse rewards in long-horizon tasks pose significant challenges for effective credit assignment \cite{andrychowicz2017hindsight,zhang2026semanticawarelogicalreasoningsemiotic}. While Process Reward Models offer step-level supervision, they rely on expensive human annotation~\cite{lightman2023let}, prompting valid alternatives such as automatic supervision via Monte Carlo estimation~\cite{uesato2022solving,wang2022self} or inference-time tree search~\cite{chen2025autobreach,chen2025red,ji2026strideedstrategygroundedstepwisereasoning,li2026arise}. However, these structural approaches primarily optimize inference rather than training dynamics. To enable persistent capability gains, self-improvement paradigms have emerged, ranging from inference-time verbal feedback in Reflexion~\cite{shinn2023reflexion} to bootstrapping rationales in STaR~\cite{zelikman2022star} and iterative RL 
updates~\citep{madani2026inclusionofthoughtsmitigatingpreferenceinstability,yang2026evotool}. Despite these advancements, existing methods typically operate either at the coarse trajectory level or strictly during inference \citep{yang2026tooltree,guo2026e3tirenhancedexperienceexploitation}.

%% file: section/3_method.tex
\section{Methodology}

We introduce \m~(Tree-structured Self-Taught Agent Rectification), a framework that addresses sparse supervision in multi-step agent RL through variance-reduced advantage estimation. Rather than modeling sampled trajectories as independent chains, \m~consolidates them into a Cognitive Tree that exposes shared decision structure and locates critical divergence steps. 
At these points, we enable agent self-rectification through thought grafting, where the agent synthesizes corrective reasoning by contrasting successful and failed branches. 

The overall pipeline is illustrated in Figure~\ref{fig:overview}.
The framework is structured around three core pillars: (1) constructing a Cognitive Tree that reveals shared reasoning structure and identifies divergence points across trajectories, (2) performing Introspective Valuation through Q-tree-based assignment, combined with In-Context Thought Grafting where the agent actively rectifies failed reasoning paths, and (3) applying Surgical Policy Optimization using the synthesized stepwise preferences. 

\subsection{Constructing the Cognitive Tree}
\label{sec:tree_construction}

We adopt the ReAct framework~\citep{yao2022react} where the agent engages in multi-turn Thought-Action-Observation cycles. At each step $t = 0, 1, \ldots, T-1$, the LLM generates a thought $z_t$ and a discrete textual action $a_t$ based on context $s_t$, 
then obtains observation $o_t$ through tool use. A complete $T$-step trajectory is:
\begin{equation}
\label{eq:trajectory}
\tau = \{(s_0, z_0, a_0, o_0), \ldots, (s_{T-1}, z_{T-1}, a_{T-1}, o_{T-1})\}
\end{equation}

GRPO~\citep{guo2025deepseek} samples $M$ trajectories $\{\tau_i\}_{i=1}^M$ per task and estimates advantages through in-group comparison $\hat{A}_{\text{GRPO}}(\tau_i) = ({R_i - \bar{R}}) / {\sigma_R}$,
where $\bar{R} = \frac{1}{M}\sum_{i=1}^M R_i$ and $\sigma_R$ are computed at each trajectory group, with sparse binary reward $R(\tau) \in \{0, 1\}$ received only upon task completion. 
There are two limitations: $\hat{A}_{\text{GRPO}}(\tau_i)$ remains constant across all steps, failing to distinguish critical decision points from routine steps; and treating trajectories as independent chains discards the opportunity to reduce variance when multiple trajectories share common reasoning prefixes before diverging.

To address these limitations, we consolidate independent chains into a Cognitive Tree $\mathcal{T} = (V, E)$ that serves two purposes: exposing shared decision structure to enable variance reduction, and locating divergence points where agent self-rectification is most valuable. Each node $v \in V$ represents a cognitive state characterized by tuple $(z, a, o)$. The tree is constructed by identifying functionally equivalent nodes across trajectories.

For each trajectory $\tau_i = \{v_0^{(i)}, v_1^{(i)}, \ldots, v_{T_i}^{(i)}\}$, we define node depth as its position in the trajectory, with $V_d$ denoting all nodes at depth $d$. The merging process is governed by two compatibility predicates. For functional equivalence, we compare policy distributions over next actions via KL divergence:
\begin{equation}
\label{eq:kl_divergence}
D_{\text{KL}}(v_i \| v_j) = \sum_a \pi_\theta(a|s_i) \log \frac{\pi_\theta(a|s_i)}{\pi_\theta(a|s_j)}
\end{equation}
Since exact computation over the entire vocabulary is intractable, we estimate this value via Monte Carlo sampling. 
Two nodes are considered functionally equivalent if the estimated $D_{\text{KL}}(v_i \| v_j) < \epsilon_{\text{kl}}$. For historical compatibility, we extract state-modifying actions $\mathcal{S}(v) = \{a \in \mathcal{H}(v) : \text{modifies\_state}(a)\}$ from node history. Two nodes are historically compatible if $\mathcal{S}(v_i) = \mathcal{S}(v_j)$, ensuring they operate on equivalent knowledge states.

At each depth $d$, we compute a compatibility graph $G_d = (V_d, E_d)$ where edge $(v_i, v_j)$ exists if both predicates hold, then merge nodes within each connected component. After merging, edge weights capture empirical transition frequencies:
\begin{equation}
\label{eq:edge_weight}
w(v \to v') = \frac{|\mathcal{T}(v \to v')|}{|\mathcal{T}(v)|}
\end{equation}
where $\mathcal{T}(v \to v')$ is the set of trajectory indices traversing this transition and $\mathcal{T}(v)$ is the set of all trajectories passing through $v$. The resulting tree $\mathcal{T}$ encodes shared decision structure: nodes with multiple children represent divergence points where different trajectories made different reasoning choices, providing the structural foundation for subsequent valuation and rectification.

\begin{figure}[!t]
    \centering
    \includegraphics[width=\columnwidth]{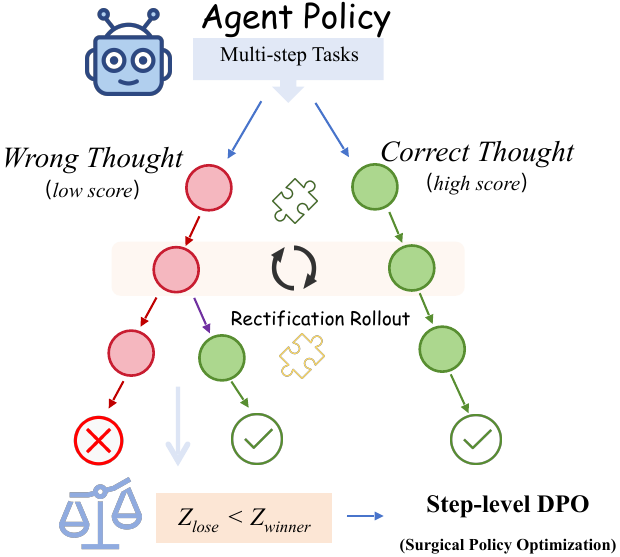}
    \caption{Thought grafting mechanism. At divergence points identified by large value spreads among children, the agent observes contrasting outcomes between successful and failed branches, then generates rectified reasoning that transfers the successful logic to the failed context. The resulting preference pairs provide step-level supervision for surgical policy optimization.}
    \label{fig:agent_grafting}
\end{figure}

\subsection{Introspective Valuation and Self-Taught Rectification}
\label{sec:valuation_rectification}

\textbf{Q-tree Valuation}
We define the Q-tree value function $Q_{\text{tree}}: V \to \mathbb{R}$ via Bellman backup:
\begin{equation}
\label{eq:qtree}
 Q_{\text{tree}}(v) = 
 \gamma \!\!\sum_{v' \in C(v)} \!\!w(v \!\to\! v') Q_{\text{tree}}(v')
\end{equation}
where $R(v) \in \{0, 1\}$ is terminal reward, $C(v)$ the children of $v$, $w(v \to v')$ the edge weights from Eq.~\eqref{eq:edge_weight}, and $Q_{\operatorname{tree}}(v)=R(v)$ if $v$ is the leaf. Based on Q-tree values, we define per-node advantages:
\begin{equation}
\label{eq:tree_advantage}
\hat{A}_{\text{tree}}(v) = \frac{Q_{\text{tree}}(v) - \bar{R}}{\sigma_R}
\end{equation}
using the same mean $\bar{R}$ and normalization ${\sigma_R}$ as GRPO. We note that $ Q_{\text{tree}}(v)$ can be computed by back-propagating the trajectory-level rewards through the tree as shown in Algorithm~1.

Consider the set of sampled trajectories $\{\tau_i\}_{i=1}^M$. For any node $v$, let $\mathcal{T}(v)$ denote the set of trajectory indices traversing $v$ (as defined in Eq.~\ref{eq:edge_weight}), with cardinality $k_v = |\mathcal{T}(v)|$. By substituting the Q-tree value definition into Eq.~\eqref{eq:tree_advantage}, we can express the node-level advantage as $\hat{A}_{\text{tree}}(v) = \frac{1}{k_v} \sum_{i \in \mathcal{T}(v)} \frac{R_i - \bar{R}}{\sigma_R}$. Assuming that trajectory completions are independent conditioned on state $v$, we derive the following properties relating node-level advantage $\hat{A}_{\text{tree}}(v)$ to GRPO relative advantage:

\begin{lemma}[Aggregation and Variance Reduction]
\label{lemma:variance_reduction}
The tree-based node advantage $\hat{A}_{\text{tree}}(v)$ satisfies:
\begin{equation}
    \hat{A}_{\text{tree}}(v) = \frac{1}{k_v} \sum_{i \in \mathcal{T}(v)} \hat{A}_{\text{GRPO}}(\tau_i),
\end{equation} and its variance satisfies:
\begin{equation}
    \text{Var}(\hat{A}_{\text{tree}}(v)) = \frac{1}{k_v} \text{Var}(\hat{A}_{\text{GRPO}}(\tau_i)).
\end{equation}
\end{lemma}

Lemma~1 shows that for nodes shared across $k$ trajectories, the tree-based advantage is the average of standard GRPO relative advantage received from different trajectories. Thus, it achieves $1/k$ variance reduction , i.e., $\text{Var}(\hat{A}_{\text{tree}}(v)) = \frac{1}{k} \text{Var}(\hat{A}_{\text{GRPO}})$, while remaining the same as GRPO for nodes belonging to a single trajectory when $k=1$. 
Hence, a more stable gradient signal for shared segments is provided. And Q-tree values also identify divergence points that nodes whose children exhibit large value differences indicate where reasoning quality determined outcomes.

\begin{figure}[!t]
    \centering
    \includegraphics[width=\columnwidth]{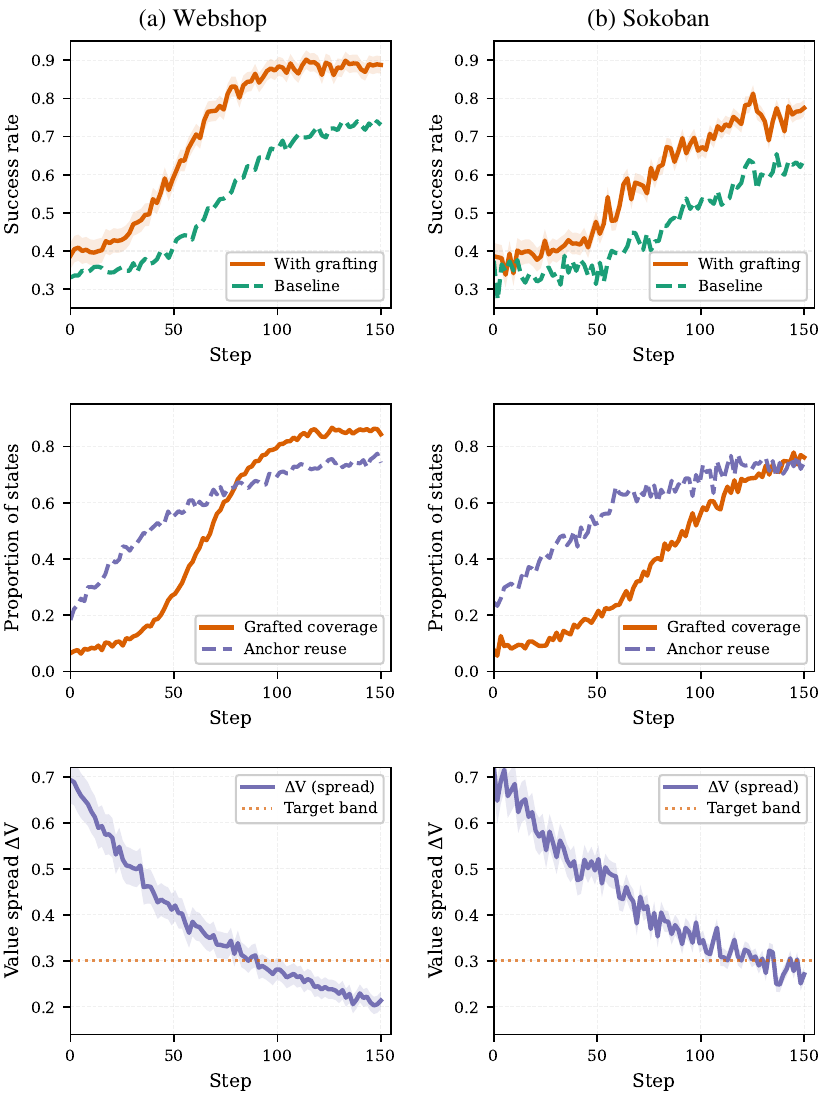}
    \caption{Training dynamics of \m~on WebShop and Sokoban, showing success rate, grafting coverage and anchor reuse, and value spread $\Delta V$ at divergence points across training iterations.}
    \label{fig:grafting_effectiveness}
\end{figure}

\textbf{Self-Taught Rectification}
Based on the cognitive tree, we identify divergent nodes as those with children exhibiting large value spreads:
\begin{equation}
\label{eq:divergence}
\Delta V(v) = \max_{v' \in C(v)} Q_{\text{tree}}(v') - \min_{v'' \in C(v)} Q_{\text{tree}}(v'') > \delta
\end{equation}
Let $V_{\text{div}}$ denote the set of divergent nodes, with $v^+, v^-$ denoting the highest-value and lowest-value children respectively. These nodes represent critical decision points where different reasoning choices led to different outcomes.

At each divergent node $v \in V_{\text{div}}$, the agent performs self-rectification through an additional rollout, as illustrated in Figure~\ref{fig:agent_grafting}. Given the shared parent context $s$ and observing that $v^+ = (z^+, a^+, o^+)$ succeeded while $v^- = (z^-, a^-, o^-)$ failed, the agent generates a rectified thought $z_{\text{rect}}$ that incorporates the successful reasoning principle from $z^+$. Since the number of divergent nodes is limited, this requires only one additional rollout pass. This creates the grafting dataset:
\begin{equation}
\label{eq:grafting_dataset}
\mathcal{D}_{\text{graft}} = \{(s, z_{\text{rect}}, z^-, t(v)) : v \in V_{\text{div}}\}
\end{equation}
where $t(v)$ is the timestep index. Each tuple provides a preference pair $(z_{\text{rect}}, z^-)$ grounded in actual outcome differences, synthesizing dense step-level supervision from sparse trajectory rewards.

Figure~\ref{fig:grafting_effectiveness} validates these mechanisms empirically. As training progresses, the value spread $\Delta V$ at divergence points decreases steadily, indicating that Q-tree valuation identifies meaningful decision boundaries where the policy gradually learns consistent behavior. Meanwhile, the increasing anchor reuse demonstrates that successful reasoning patterns discovered through grafting transfer effectively across similar contexts, confirming that thought grafting synthesizes generalizable corrections rather than instance-specific fixes.

\begin{algorithm}[!t]
\caption{\m~Training Procedure}
\label{alg:tstar}
\begin{algorithmic}[1]
\REQUIRE Policy $\pi_\theta$, tasks $\mathcal{P}$, group size $M$, thresholds $\delta$, $\epsilon_{\text{kl}}$
\STATE Initialize $\pi_{\text{ref}} \gets \pi_\theta$, $\mathcal{D}_{\text{graft}} \gets \emptyset$
\FOR{iteration $k = 1$ to $K$}
    \FOR{task $p \sim \mathcal{P}$}
        \STATE Sample $\{\tau_i\}_{i=1}^M \sim \pi_\theta(\cdot | p)$
        \STATE Build $\mathcal{T}$ via Eq.~\eqref{eq:kl_divergence}-\eqref{eq:edge_weight}, compute $Q_{\text{tree}}$ via Eq.~\eqref{eq:qtree}
        \FOR{$v \in \mathcal{T}$ where $\Delta V(v) > \delta$}
            \STATE Generate $z_{\text{rect}} \sim \pi_\theta(\cdot | s, v^+, v^-)$
            \STATE $\mathcal{D}_{\text{graft}} \gets \mathcal{D}_{\text{graft}} \cup \{(s, z_{\text{rect}}, z^-)\}$
        \ENDFOR
    \ENDFOR
    \STATE Update $\theta$ via Eq.~\eqref{eq:hybrid_objective}, $\pi_{\text{ref}} \gets \alpha \pi_{\text{ref}} + (1-\alpha) \pi_\theta$
\ENDFOR
\end{algorithmic}
\end{algorithm}

\begin{table*}[!t]
\centering
\small
\setlength{\tabcolsep}{3.5pt}
\caption{Performance on interactive and embodied tasks. Success rate (\%) for ALFWorld subtasks and score/success for WebShop. \m~consistently outperforms across all baselines and architectures.}
\label{tab:interactive_tasks}
\begin{tabular}{ll ccccccc | cc}
\toprule
\multirow{2}{*}{Type} & \multirow{2}{*}{Method} & \multicolumn{7}{c|}{ALFWorld (Success Rate \%)} & \multicolumn{2}{c}{WebShop} \\
\cmidrule(lr){3-9} \cmidrule(lr){10-11}
 & & Pick & Look & Clean & Heat & Cool & Pick2 & All & Score & Succ. \\
\midrule
\multicolumn{11}{l}{\textit{Closed-Source Model}} \\
Prompting & GPT-4o & 86.8 & 74.4 & 79.1 & 90.9 & 93.6 & 83.1 & 84.6 & 84.2 & 61.7 \\
Prompting & Gemini-1.5-Pro & 81.6 & 68.1 & 72.9 & 84.8 & 88.6 & 77.7 & 79.0 & 79.3 & 54.7 \\
\midrule
\multicolumn{11}{l}{\textit{Qwen2.5-3B-Instruct}} \\
\multirow{2}{*}{Prompting} & ReAct & 48.2 & 40.9 & 44.8 & 51.7 & 55.4 & 38.6 & 46.6 & 50.7 & 25.3 \\
 & Reflexion & 61.2 & 49.6 & 55.9 & 65.2 & 68.4 & 45.6 & 57.7 & 58.8 & 32.3 \\
\cmidrule{2-11}
RL Training & GRPO & 81.1 & 69.9 & 75.7 & 84.7 & 88.1 & 73.4 & 78.8 & 64.2 & 36.2 \\
\rowcolor{ourcolor}
RL Training & \quad + \m & 82.4 & \textbf{75.1} & 77.8 & 86.7 & 89.2 & \textbf{78.9} & \textbf{81.7} & \textbf{68.2} & \textbf{39.6} \\
\cmidrule{2-11}
RL Training & DAPO & 84.3 & 73.4 & 79.1 & 87.1 & 91.1 & 76.7 & 82.0 & 67.8 & 39.1 \\
\rowcolor{ourcolor}
RL Training & \quad + \m & \textbf{86.7} & 77.2 & \textbf{82.8} & 89.1 & 92.6 & 81.9 & \textbf{85.0} & \textbf{71.7} & \textbf{42.4} \\
\cmidrule{2-11}
RL Training & GiGPO & 90.1 & 78.9 & 85.6 & 92.1 & 95.2 & 83.8 & 87.6 & 73.1 & 41.6 \\
\rowcolor{ourcolor}
RL Training & \quad + \m & 91.1 & \textbf{83.2} & 87.8 & 94.1 & 96.3 & \textbf{87.9} & \textbf{90.1} & \textbf{76.3} & \textbf{45.1} \\
\midrule
\multicolumn{11}{l}{\textit{Phi-4-mini-instruct-3.8B}} \\
\multirow{2}{*}{Prompting} & ReAct & 46.2 & 39.2 & 42.3 & 49.6 & 53.2 & 37.3 & 44.6 & 49.8 & 23.8 \\
 & Reflexion & 58.6 & 47.1 & 54.2 & 62.1 & 66.2 & 42.9 & 55.2 & 56.4 & 30.7 \\
\cmidrule{2-11}
RL Training & GRPO & 79.7 & 67.3 & 73.3 & 82.9 & 86.8 & 70.8 & 76.8 & 62.3 & 34.1 \\
\rowcolor{ourcolor}
RL Training & \quad + \m & 81.3 & \textbf{73.2} & 76.8 & 84.4 & 88.1 & \textbf{76.8} & \textbf{80.1} & \textbf{65.2} & \textbf{38.3} \\
\cmidrule{2-11}
RL Training & DAPO & 82.6 & 71.9 & 77.9 & 85.9 & 89.9 & 75.7 & 80.7 & 66.1 & 37.4 \\
\rowcolor{ourcolor}
RL Training & \quad + \m & \textbf{85.9} & 75.6 & \textbf{81.9} & 87.8 & 91.2 & 80.6 & \textbf{83.8} & \textbf{69.2} & \textbf{41.1} \\
\cmidrule{2-11}
RL Training & GiGPO & 88.7 & 77.3 & 84.2 & 90.6 & 94.2 & 82.1 & 86.2 & 71.3 & 41.1 \\
\rowcolor{ourcolor}
RL Training & \quad + \m & 90.1 & \textbf{81.7} & 87.3 & 92.7 & 95.7 & \textbf{86.8} & \textbf{89.0} & \textbf{74.2} & \textbf{44.3} \\
\bottomrule
\end{tabular}
\end{table*}

\begin{figure*}[!t]
    \centering
    \includegraphics[width=1.0\linewidth]{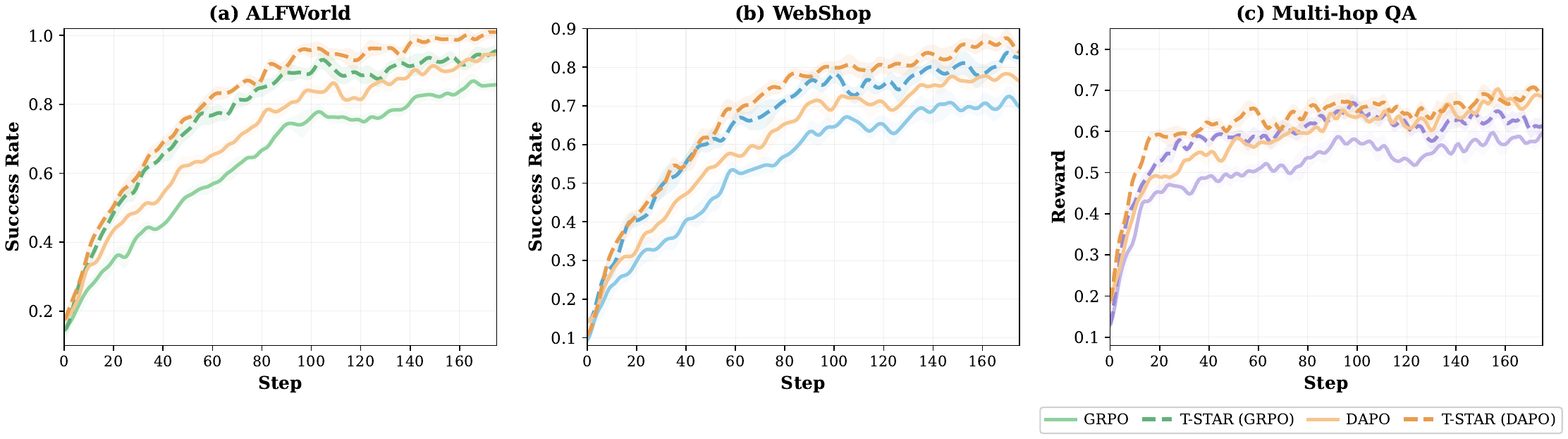}
    \caption{Training dynamics on (a) ALFWorld, (b) WebShop, and (c) Multi-hop QA. \m~achieves faster convergence and higher final performance. Shaded regions indicate standard deviation across three seeds.}
    \label{fig:training_dynamics}
\end{figure*}

\textbf{Surgical Policy Optimization}
The divergent nodes identified in the cognitive tree represent critical decision points where reasoning quality determined outcomes, which lead to the comparative structure.
The preference pairs $(z_{\text{rect}}, z^-)$ constructed at these points provide targeted supervision for policy improvement.
Therefore based on that, we optimize the overall loss function through a hybrid objective combining trajectory-level and step-level learning:
\vspace{-2pt}
\begin{equation}
\label{eq:hybrid_objective}
\mathcal{L}(\theta) = \mathcal{L}_{\text{GRPO}}(\theta) + \lambda \mathcal{L}_{\text{Surgical}}(\theta)
\end{equation}
where $\lambda$ controls the relative weight. The surgical loss implements preference learning via the Bradley-Terry model:
\begin{equation}
\label{eq:surgical_loss}
\mathcal{L}_{\text{Surgical}}(\theta) = -\mathbb{E}_{(s, z_{\text{rect}}, z^-) \sim \mathcal{D}_{\text{graft}}} \left[ \log \sigma(\beta \Delta_\theta) \right]
\end{equation}
with preference margin:
\begin{equation}
\label{eq:preference_margin}
\Delta_\theta = \log \frac{\pi_\theta(z_{\text{rect}}|s)}{\pi_{\text{ref}}(z_{\text{rect}}|s)} - \log \frac{\pi_\theta(z^-|s)}{\pi_{\text{ref}}(z^-|s)}
\end{equation}

To ensure surgical precision, where step-level optimization always plays a supporting role, gradients from $\mathcal{L}_{\text{Surgical}}$ are masked to affect only the divergence timestep $t_{\text{div}}$, while $\mathcal{L}_{\text{GRPO}}$ provides learning signal across all timesteps. 
The reference policy $\pi_{\text{ref}}$ is updated via exponential moving average to prevent distribution drift. 

Algorithm~\ref{alg:tstar} summarizes the complete training procedure. In \m~framework, the agent samples trajectories, constructs the cognitive tree to identify divergence points, operates rectified thoughts and then updates the designed policy.

%% file: section/4_result.tex
\section{Results}

\begin{table*}[!t]
\centering
\small
\setlength{\tabcolsep}{3.5pt}
\caption{Performance on search-augmented QA tasks (Exact Match). \m~shows largest gains on multi-hop reasoning. Best performance (highlighted in bold) is always achieved by \m.}
\label{tab:reasoning_qa}
\begin{tabular}{ll cccc | cccc | c}
\toprule
\multirow{2}{*}{Type} & \multirow{2}{*}{Method} & \multicolumn{4}{c|}{Single-Hop QA (EM)} & \multicolumn{4}{c|}{Multi-Hop QA (EM)} & \multirow{2}{*}{Avg.} \\
\cmidrule(lr){3-6} \cmidrule(lr){7-10}
 & & NQ & Trivia & PopQA & Avg. & Hotpot & 2Wiki & MusiQ & Bamb & \\
\midrule
\multicolumn{11}{l}{\textit{Closed-Source Model}} \\
Prompting & GPT-4o & 59.5 & 73.0 & 61.5 & 64.7 & 54.0 & 50.5 & 47.0 & 44.5 & 55.8 \\
Prompting & Gemini-1.5-Pro & 56.0 & 69.5 & 57.5 & 61.0 & 50.0 & 46.5 & 43.0 & 40.5 & 52.0 \\
\midrule
\multicolumn{11}{l}{\textit{Qwen2.5-3B-Instruct}} \\
\multirow{3}{*}{Prompting} & Direct & 23.5 & 36.0 & 25.0 & 28.2 & 19.0 & 16.5 & 13.0 & 11.5 & 21.2 \\
 & ReAct & 29.0 & 43.5 & 30.5 & 34.3 & 26.0 & 22.5 & 19.5 & 17.5 & 26.7 \\
 & Search-o1 & 32.5 & 46.5 & 33.5 & 37.5 & 28.5 & 25.5 & 21.5 & 19.5 & 29.6 \\
\cmidrule{2-11}
RL Training & GRPO & 41.0 & 54.5 & 41.5 & 45.7 & 36.5 & 32.0 & 28.0 & 25.5 & 37.1 \\
\rowcolor{ourcolor}
RL Training & \quad + \m & 43.5 & 56.0 & 43.5 & \textbf{47.7} & \textbf{40.5} & \textbf{35.5} & \textbf{30.5} & 28.0 & \textbf{39.7} \\
\cmidrule{2-11}
RL Training & DAPO & 43.0 & 56.5 & 43.0 & 47.5 & 38.5 & 34.5 & 30.0 & 27.5 & 39.1 \\
\rowcolor{ourcolor}
RL Training & \quad + \m & \textbf{45.5} & 58.0 & 45.0 & \textbf{49.5} & \textbf{42.0} & \textbf{38.0} & \textbf{33.0} & \textbf{30.5} & \textbf{41.8} \\
\cmidrule{2-11}
RL Training & GiGPO & 45.0 & 58.5 & 44.5 & 49.3 & 41.0 & 37.0 & 32.5 & 30.0 & 41.1 \\
\rowcolor{ourcolor}
RL Training & \quad + \m & 47.0 & \textbf{60.5} & \textbf{46.5} & \textbf{51.3} & \textbf{45.0} & \textbf{41.5} & \textbf{35.5} & \textbf{33.0} & \textbf{44.1} \\
\midrule
\multicolumn{11}{l}{\textit{Phi-4-mini-instruct-3.8B}} \\
\multirow{3}{*}{Prompting} & Direct & 25.5 & 38.5 & 27.5 & 30.5 & 20.5 & 17.5 & 14.0 & 12.0 & 22.4 \\
 & ReAct & 31.5 & 46.0 & 33.0 & 36.8 & 27.5 & 24.0 & 21.0 & 19.0 & 28.6 \\
 & Search-o1 & 35.0 & 49.0 & 36.0 & 40.0 & 30.0 & 27.0 & 23.0 & 21.0 & 31.5 \\
\cmidrule{2-11}
RL Training & GRPO & 43.5 & 57.5 & 44.0 & 48.3 & 38.5 & 34.0 & 30.0 & 27.5 & 39.4 \\
\rowcolor{ourcolor}
RL Training & \quad + \m & 46.0 & 59.0 & 46.0 & \textbf{50.3} & \textbf{42.5} & \textbf{37.5} & \textbf{33.0} & 30.0 & \textbf{42.0} \\
\cmidrule{2-11}
RL Training & DAPO & 45.5 & 59.5 & 46.0 & 50.3 & 40.5 & 36.5 & 32.0 & 29.5 & 41.2 \\
\rowcolor{ourcolor}
RL Training & \quad + \m & \textbf{48.0} & 61.0 & 48.0 & \textbf{52.3} & \textbf{44.5} & \textbf{40.0} & \textbf{35.0} & \textbf{32.5} & \textbf{44.0} \\
\cmidrule{2-11}
RL Training & GiGPO & 47.5 & 61.5 & 47.5 & 52.2 & 43.5 & 39.0 & 34.5 & 32.0 & 43.6 \\
\rowcolor{ourcolor}
RL Training & \quad + \m & 49.5 & \textbf{63.5} & \textbf{49.5} & \textbf{54.2} & \textbf{48.0} & \textbf{43.5} & \textbf{37.5} & \textbf{35.0} & \textbf{46.6} \\
\bottomrule
\end{tabular}
\end{table*}

\begin{table}[!t]
\centering
\small
\caption{Performance on logical planning tasks. \m~enables recovery from dead-ends through thought grafting.}
\label{tab:logical_planning}
\resizebox{\columnwidth}{!}{%
\begin{tabular}{ll ccc | cc}
\toprule
\multirow{2}{*}{Type} & \multirow{2}{*}{Method} & \multicolumn{3}{c|}{Sokoban (Success \%)} & \multicolumn{2}{c}{Blocksworld} \\
\cmidrule(lr){3-5} \cmidrule(lr){6-7}
 & & Easy & Med. & Hard & Stack & Unstack \\
\midrule
\multicolumn{7}{l}{\textit{Closed-Source Model}} \\
Prompting & GPT-4o & 78.3 & 45.4 & 17.4 & 84.6 & 89.8 \\
Prompting & Gemini-1.5-Pro & 72.7 & 40.4 & 14.4 & 79.7 & 85.4 \\
\midrule
\multicolumn{7}{l}{\textit{Qwen2.5-3B-Instruct}} \\
RL Training & PPO & 16.9 & 4.7 & 0.7 & 24.1 & 32.4 \\
\cmidrule{2-7}
RL Training & GRPO & 37.4 & 13.6 & 2.2 & 47.1 & 54.4 \\
\rowcolor{ourcolor}
RL Training & \quad + \m & \textbf{44.7} & \textbf{20.6} & \textbf{6.7} & \textbf{55.7} & \textbf{62.9} \\
\cmidrule{2-7}
RL Training & DAPO & 42.6 & 18.4 & 3.9 & 52.3 & 58.9 \\
\rowcolor{ourcolor}
RL Training & \quad + \m & \textbf{49.4} & \textbf{25.9} & \textbf{8.3} & \textbf{61.4} & \textbf{67.7} \\
\cmidrule{2-7}
RL Training & GiGPO & 47.2 & 22.8 & 6.2 & 58.3 & 64.8 \\
\rowcolor{ourcolor}
RL Training & \quad + \m & \textbf{54.4} & \textbf{29.6} & \textbf{11.2} & \textbf{67.1} & \textbf{73.4} \\
\midrule
\multicolumn{7}{l}{\textit{Phi-4-mini-instruct-3.8B}} \\
RL Training & PPO & 18.7 & 5.2 & 0.4 & 25.9 & 33.9 \\
\cmidrule{2-7}
RL Training & GRPO & 39.4 & 14.8 & 3.3 & 48.8 & 56.1 \\
\rowcolor{ourcolor}
RL Training & \quad + \m & \textbf{46.9} & \textbf{22.4} & \textbf{6.7} & \textbf{58.4} & \textbf{64.7} \\
\cmidrule{2-7}
RL Training & DAPO & 43.6 & 19.1 & 4.2 & 53.9 & 59.8 \\
\rowcolor{ourcolor}
RL Training & \quad + \m & \textbf{51.8} & \textbf{27.4} & \textbf{9.4} & \textbf{63.4} & \textbf{70.3} \\
\cmidrule{2-7}
RL Training & GiGPO & 48.6 & 23.8 & 7.3 & 59.6 & 66.8 \\
\rowcolor{ourcolor}
RL Training & \quad + \m & \textbf{56.1} & \textbf{32.4} & \textbf{12.6} & \textbf{69.8} & \textbf{76.4} \\
\bottomrule
\end{tabular}%
}
\end{table}

\subsection{Experiment Setup}
The evaluation spans three task categories with diverse reasoning requirements. For embodied and interactive environments, ALFWorld provides text-based household tasks (pick, clean, heat, cool) while WebShop requires e-commerce navigation with product search and attribute verification. Search-augmented QA includes single-hop datasets (Natural Questions, TriviaQA, PopQA) and multi-hop benchmarks (HotpotQA, 2WikiMultiHopQA, MusiQue, Bamboogle) requiring compositional reasoning across documents. Logical planning tasks include Sokoban with varying difficulty levels and Blocksworld with stacking operations.

Baselines include three group-based RL methods: GRPO with in-group trajectory comparison, DAPO with dynamic advantage scaling, and GiGPO with group-in-group optimization structures. Prompting methods (ReAct, Reflexion) and closed-source models (GPT-4o, Gemini-1.5-Pro) serve as additional references. \m~is applied on top of each RL baseline using Qwen2.5-3B-Instruct and Phi-4-mini-instruct-3.8B, with all methods adopting the ReAct framework and search tool access. Training runs for 160 steps with EMA-updated reference policy ($\alpha = 0.95$). Evaluation uses success rate for ALFWorld and Sokoban, score and success rate for WebShop, and Exact Match for QA tasks, with results averaged over three seeds.

\subsection{Main Results}
\m~demonstrates consistent improvements across all three task categories and baseline methods, as shown in Tables~\ref{tab:interactive_tasks}--\ref{tab:logical_planning}. On interactive and embodied tasks, \m~achieves 3.0--3.8\% gains on ALFWorld and 3.2--5.8\% on WebShop. The improvement is particularly pronounced on complex subtasks requiring longer action sequences, such as Pick2 and Clean, where trajectory overlap occurs more frequently and enables greater variance reduction through shared node averaging. On search-augmented QA, \m~shows its largest gains on multi-hop reasoning: 2.8--7.5\% on HotpotQA, 2.8--6.9\% on 2WikiMultiHopQA, 2.2--4.6\% on MusiQue, and 2.8--6.8\% on Bamboogle, while single-hop improvements are comparatively modest at 1.9--3.5\%. This disparity aligns with our theoretical analysis---multi-hop tasks involve more sequential decisions, creating more opportunities for trajectory sharing at early reasoning steps. On logical planning tasks, \m~achieves 5.5--7.5\% gains on easy levels, 4.5--8.5\% on medium difficulty, and 3.0--4.0\% even on hard instances where baseline methods struggle to learn meaningful policies.
The training dynamics in Figure~\ref{fig:training_dynamics} also reveal that \m~not only achieves higher final performance but also exhibits substantially more stable learning. 
The stability is evident in multi-hop QA, where sparse rewards cause baseline methods to exhibit high variance and occasional performance drops during training, while our method maintains stable improvement throughout training process.

\subsection{Analysis}

\begin{figure}[!t]
    \centering
    \includegraphics[width=1.0\linewidth]{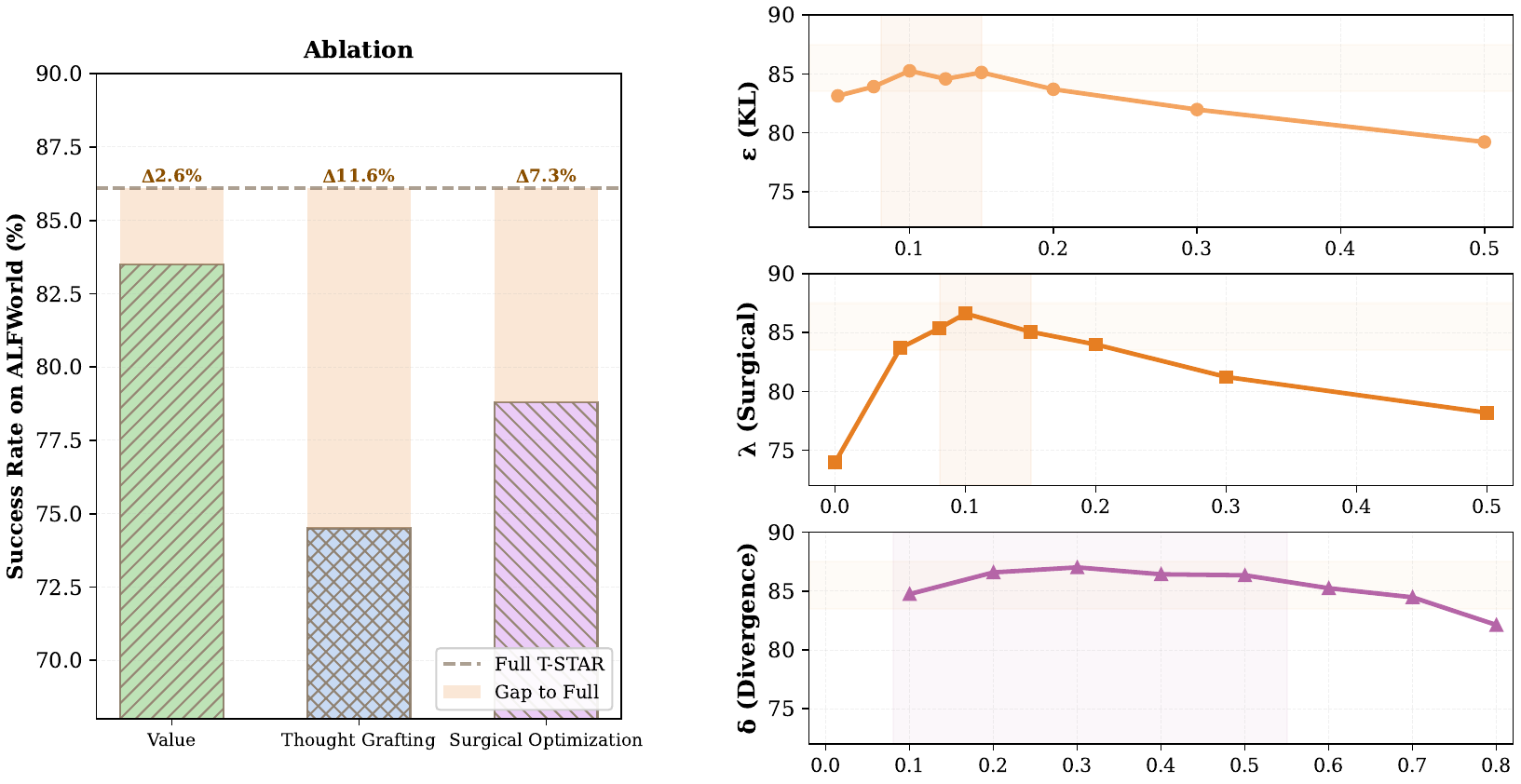}
    \caption{Ablation study on ALFWorld. Left: component contribution. Right: hyperparameter sensitivity ($\epsilon_{\text{kl}}$, $\lambda$, $\delta$).}
    \label{fig:ablation}
\end{figure}

As illustrated in Figure~\ref{fig:ablation}, the ablation experiments on ALFWorld reveal the relative contribution of each component to the overall improvement. Removing thought grafting causes the largest performance drop at 11.6\%, confirming that corrective supervision at divergence points is the most critical mechanism. Removing Q-tree valuation reduces performance by 7.3\%, demonstrating that variance-reduced credit assignment provides substantial benefit even without explicit rectification. Removing surgical optimization decreases performance by 2.6\%, indicating that while step-level preference learning contributes, the primary gains derive from improved credit assignment and corrective data synthesis rather than the specific optimization procedure. The hyperparameter sensitivity analysis in Figure~\ref{fig:ablation} shows robustness across reasonable ranges of $\epsilon_{\text{kl}}$, $\lambda$, and $\delta$, with graceful degradation outside optimal settings.

The relationship between task complexity and improvement magnitude provides additional validation for our approach. Tasks with longer reasoning chains exceeding 15 steps show 5--8\% improvements, while shorter tasks under 10 steps yield 2--4\% gains, confirming that tree-based credit assignment provides greatest value when sparse rewards must propagate through many decision points. Manual inspection of 100 grafted thoughts reveals that 83\% provide semantically meaningful corrections that address specific reasoning errors, including adding missing search constraints, correcting logical inference steps, and refining action selection based on observed outcomes.

%% file: section/5_conclusion.tex
\section{Conclusion}
In this work, we presented \m, a framework that addresses sparse supervision in multi-step agent reinforcement learning through tree-structured credit assignment and self-taught rectification. By consolidating independent rollouts into a Cognitive Tree, \m~enables variance-reduced advantage estimation while localizing critical divergence points where the agent synthesizes step-level corrective supervision through thought grafting and surgical policy optimization, without requiring additional reward models or rollouts.

Experiments across embodied, interactive, reasoning, and planning tasks demonstrate consistent improvements over a wide range of baselines, with gains pronounced on tasks requiring extended reasoning chains. 
The decreasing value spread at divergence points and increasing anchor reuse confirm that the framework produces genuine policy improvement through generalizable corrections. The core insight that independent rollouts contain latent shared structure exploitable for both variance reduction and targeted correction, suggests broader applications to sequential decision-making settings.

%% file: section/x_appendix.tex
\clearpage
\appendix

\section*{Appendix}
\section{Further Analysis}

\subsection{Theoretical Analysis}
\label{app:theory}
\begin{lemma}[Aggregation and Variance Reduction]
The tree-based node advantage $\hat{A}_{\text{tree}}(v)$ satisfies:
\begin{equation}
    \hat{A}_{\text{tree}}(v) = \frac{1}{k_v} \sum_{i \in \mathcal{T}(v)} \hat{A}_{\text{GRPO}}(\tau_i),
\end{equation} and
\begin{equation}
    \text{Var}(\hat{A}_{\text{tree}}(v)) = \frac{1}{k_v} \text{Var}(\hat{A}_{\text{GRPO}}(\tau_i)).
\end{equation}
\end{lemma}

\begin{proof}
\textbf{1. Derivation of Unbiased Aggregation}

Let $R_i \triangleq R(\tau_i)$ denote the terminal reward for trajectory $i$. The GRPO advantage is given by $\hat{A}_{\text{GRPO}}(\tau_i) = \sigma_R^{-1}(R_i - \bar{R})$.
Substituting the definition of $Q_{\text{tree}}(v)$ into the node advantage formulation:
\begin{equation}
\begin{split}
    \hat{A}_{\text{tree}}(v) &= \frac{1}{\sigma_R} \left( \left[ \frac{1}{k_v} \sum_{i \in \mathcal{T}(v)} R_i \right] - \bar{R} \right) \\
    &= \frac{1}{\sigma_R} \left( \frac{1}{k_v} \sum_{i \in \mathcal{T}(v)} R_i - \frac{1}{k_v} \sum_{i \in \mathcal{T}(v)} \bar{R} \right) \\
    &= \frac{1}{k_v} \sum_{i \in \mathcal{T}(v)} \frac{R_i - \bar{R}}{\sigma_R} \\
    &= \frac{1}{k_v} \sum_{i \in \mathcal{T}(v)} \hat{A}_{\text{GRPO}}(\tau_i).
\end{split}
\end{equation}
This establishes the linear equivalence between the node-level advantage and the mean trajectory-level advantage.

\textbf{2. Variance Analysis via Conditional Independence}

Let $X_i$ be the random variable representing the advantage $\hat{A}_{\text{GRPO}}(\tau_i)$ for a trajectory passing through $v$. We analyze the variance of the estimator $\hat{A}_{\text{tree}}(v)$ conditioned on the prefix node $v$.
Since the policy $\pi_\theta$ generates rollouts stochastically, for any distinct pair $i, j \in \mathcal{T}(v)$ with $i \neq j$, the completions are independent given $v$. Thus, the covariance term vanishes:
\begin{equation}
    \text{Cov}(X_i, X_j \mid v) = 0, \quad \forall i \neq j.
\end{equation}
Let $\Sigma^2 \triangleq \text{Var}(X_i)$ be the variance of the single-trajectory advantage. Expanding the variance of the sum:
\begin{equation}
\begin{split}
    &\text{Var}(\hat{A}_{\text{tree}}(v)) = \text{Var}\left( \frac{1}{k_v} \sum_{i \in \mathcal{T}(v)} X_i \right) \\
    &= \frac{1}{k_v^2} \left( \sum_{i \in \mathcal{T}(v)} \text{Var}(X_i) + \sum_{i \neq j} \text{Cov}(X_i, X_j) \right) \\
    &= \frac{1}{k_v^2} \left( \sum_{i=1}^{k_v} \Sigma^2 + 0 \right) \\
    &= \frac{k_v \Sigma^2}{k_v^2} = \frac{1}{k_v} \text{Var}(\hat{A}_{\text{GRPO}}(\tau_i)).
\end{split}
\end{equation}
Consequently, for any shared node where $k_v > 1$, the variance of the gradient estimation signal is strictly reduced.
\end{proof}

\subsection{Complexity Analysis}

We analyze the computational overhead of \m~relative to standard GRPO, decomposing into three components and providing both theoretical bounds and empirical measurements.

\subsubsection{Tree Construction}

The tree construction phase involves two sub-operations: functional equivalence testing and graph-based merging.

\textbf{Functional Equivalence Testing.} For $M$ trajectories of average length $T$, we must evaluate pairwise KL divergence at each depth $d \in \{0, 1, \ldots, T-1\}$. At depth $d$, let $n_d$ denote the number of distinct nodes before merging. Computing exact KL divergence $D_{\text{KL}}(v_i \| v_j) = \sum_a \pi_\theta(a|s_i) \log \frac{\pi_\theta(a|s_i)}{\pi_\theta(a|s_j)}$ requires enumerating the action space $|A|$, which is intractable for large vocabularies.

We adopt Monte Carlo approximation: sample $K$ actions from $\pi_\theta(\cdot|s_i)$ and estimate
\begin{equation}
\hat{D}_{\text{KL}}(v_i \| v_j) = \frac{1}{K} \sum_{k=1}^{K} \log \frac{\pi_\theta(a_k|s_i)}{\pi_\theta(a_k|s_j)}, \quad a_k \sim \pi_\theta(\cdot|s_i).
\end{equation}
This requires $K$ forward passes per node pair. With $n_d \leq M$ nodes at each depth, the total cost across all depths is $O(M^2 T K)$ forward passes. We set $K=16$, which provides sufficient accuracy for threshold-based equivalence decisions.

\textbf{Graph-Based Merging.} At each depth, we construct a compatibility graph $G_d = (V_d, E_d)$ where edges connect functionally equivalent and historically compatible nodes. Finding connected components via union-find requires $O(n_d^2 \cdot \alpha(n_d))$ operations, where $\alpha$ is the inverse Ackermann function. Summing across depths yields $O(M^2 T \cdot \alpha(M))$, which is effectively linear in practice.

\textbf{Historical Compatibility.} Checking $S(v_i) = S(v_j)$ requires comparing action histories of length at most $d$. With efficient hashing of state-modifying action sequences, this adds $O(M^2 T^2)$ hash comparisons in the worst case, though average-case complexity is lower due to early rejection.

\subsubsection{Q-tree Computation}

The Q-tree value computation (Eq. 5) performs a single bottom-up traversal of the constructed tree. Let $|V|$ denote the total number of nodes after merging.

\textbf{Bellman Backup.} Each internal node $v$ computes $Q_{\text{tree}}(v) = \gamma \sum_{v' \in C(v)} w(v \to v') Q_{\text{tree}}(v')$, requiring $O(|C(v)|)$ operations. Summing over all nodes:
\begin{equation}
\sum_{v \in V} |C(v)| = |E| \leq |V| - 1 = O(MT).
\end{equation}
Thus, the total Bellman backup cost is $O(MT)$ arithmetic operations.

\textbf{Edge Weight Computation.} Edge weights $w(v \to v') = |T(v \to v')| / |T(v)|$ are computed during tree construction by maintaining trajectory membership sets. Using efficient set operations, this adds $O(MT)$ overhead.

\textbf{Divergence Identification.} Computing $\Delta V(v) = \max_{v'} Q_{\text{tree}}(v') - \min_{v''} Q_{\text{tree}}(v'')$ for each internal node requires $O(|C(v)|)$ comparisons per node, totaling $O(MT)$ across the tree.

\subsubsection{Thought Grafting}

For each divergent node $v \in V_{\text{div}}$, generating a rectified thought requires one forward pass through the policy model.

\textbf{Number of Divergent Nodes.} The divergence threshold $\delta$ controls the granularity of rectification. Let $p_{\text{div}}$ denote the fraction of internal nodes satisfying $\Delta V(v) > \delta$. With $\delta = 0.3$, we observe $p_{\text{div}} = 0.12$ on ALFWorld and $p_{\text{div}} = 0.09$ on WebShop. With approximately $MT/2$ internal nodes, we have $|V_{\text{div}}| \approx p_{\text{div}} \cdot MT / 2$.

\textbf{Generation Cost.} Each rectified thought generation involves conditioning on the shared context $s$, the successful branch $v^+$, and the failed branch $v^-$. The input length is approximately $3L$ where $L$ is the average context length. Generation produces thoughts of average length $l_z$, requiring $O(l_z)$ autoregressive steps.

\subsubsection{Overall Complexity Comparison}

Table~\ref{tab:complexity} summarizes the computational complexity of each component.

\begin{table}[h]
\centering
\small
\begin{tabular}{lcc}
\toprule
Component & Time Complexity & Dominant Cost \\
\midrule
Standard GRPO & $O(MT \cdot L)$ & Forward/backward passes \\
\midrule
\multicolumn{3}{l}{\textit{\m~Additional Overhead}} \\
KL Estimation & $O(M^2 T K)$ & Forward passes (inference) \\
Graph Merging & $O(M^2 T)$ & CPU operations \\
Q-tree Backup & $O(MT)$ & Arithmetic operations \\
Thought Grafting & $O(p_{\text{div}} \cdot MT \cdot l_z)$ & Forward passes (generation) \\
\bottomrule
\end{tabular}
\caption{Computational complexity breakdown. $M$: trajectories per task, $T$: average trajectory length, $K$: MC samples, $L$: context length, $l_z$: thought length, $p_{\text{div}}$: divergence fraction.}
\label{tab:complexity}
\end{table}

The dominant additional cost is KL estimation during tree construction. However, this involves inference-only forward passes without gradient computation, which are significantly cheaper than training forward-backward passes.

\subsubsection{Empirical Runtime Analysis}

We measure actual training time on ALFWorld with Qwen2.5-3B-Instruct.

\begin{table}[h]
\centering
\small
\begin{tabular}{lccc}
\toprule
Method & Time/Iter (s) & Total Time (h) & Relative \\
\midrule
GRPO & 42.3 & 1.88 & 1.00$\times$ \\
GRPO + \m & 51.8 & 2.30 & 1.22$\times$ \\
\midrule
\multicolumn{4}{l}{\textit{\m~Breakdown}} \\
Tree Construction & 5.2 & -- & -- \\
Q-tree + Divergence & 0.8 & -- & -- \\
Thought Grafting & 3.5 & -- & -- \\
\bottomrule
\end{tabular}
\caption{Runtime comparison on ALFWorld (160 training iterations).}
\label{tab:runtime}
\end{table}

\begin{table*}[!t]
\centering
\small
\begin{tabular}{lccc}
\toprule
Model Size & GRPO Time (s) & \m~Overhead (s) & Relative Overhead \\
\midrule
1.5B & 28.4 & 8.2 & 28.9\% \\
3B & 42.3 & 9.5 & 22.5\% \\
7B & 78.6 & 11.3 & 14.4\% \\
\bottomrule
\end{tabular}
\caption{Scaling behavior of \m~overhead with model size on ALFWorld.}
\label{tab:model_scaling}
\end{table*}

\subsubsection{Scalability Considerations}

\textbf{Scaling with Trajectory Count $M$.} The $O(M^2)$ term in tree construction becomes significant for large $M$. For $M > 16$, we recommend hierarchical clustering: first group trajectories by coarse features (e.g., first action), then apply full equivalence testing within groups. This reduces the complexity to $O(M^2/g + g \cdot (M/g)^2) = O(M^2/g)$ where $g$ is the number of groups.

\textbf{Scaling with Trajectory Length $T$.} Longer trajectories increase tree depth but typically exhibit more sharing at early depths, maintaining reasonable $|V|$. The linear dependence on $T$ in most components ensures scalability. Empirically, we observe that the effective tree size grows sublinearly with $T$ due to increased merging opportunities at shallow depths.

\textbf{Scaling with Model Size.} Forward pass cost scales with model parameters, affecting both KL estimation and thought grafting. For larger models, the relative overhead of \m~decreases since the fixed-cost components (graph operations, Q-tree computation) remain constant. Table~\ref{tab:model_scaling} presents scaling behavior across different model sizes.

\section{Experimental Details}
\label{app:exp_details}

\subsection{Dataset Statistics}

We evaluate \m~across 11 datasets spanning four task categories. Table~\ref{tab:dataset_stats} summarizes the statistics of all benchmarks. The Avg. Steps column indicates the average number of reasoning steps required per task, which directly correlates with the potential for trajectory sharing in our cognitive tree construction.

\begin{table}[t]
\centering
\small
\setlength{\tabcolsep}{3pt}
\begin{tabular}{llrrr}
\toprule
Category & Dataset & Train & Test & Steps \\
\midrule
\multirow{2}{*}{Embodied} & ALFWorld & 3,321 & 140 & 12.4 \\
& WebShop & 10,587 & 500 & 8.6 \\
\midrule
\multirow{3}{*}{Single-hop} & NQ & 79,168 & 3,610 & 3.2 \\
& TriviaQA & 78,785 & 11,313 & 2.8 \\
& PopQA & -- & 14,267 & 2.5 \\
\midrule
\multirow{4}{*}{Multi-hop} & HotpotQA & 90,447 & 7,405 & 6.4 \\
& 2WikiMH & 15,000 & 12,576 & 5.8 \\
& MusiQue & 19,938 & 2,417 & 7.2 \\
& Bamboogle & -- & 125 & 4.6 \\
\midrule
\multirow{2}{*}{Planning} & Sokoban & 5,000 & 1,000 & 18.5 \\
& Blocksworld & 4,000 & 800 & 14.2 \\
\bottomrule
\end{tabular}
\caption{Dataset statistics across four task categories.}
\label{tab:dataset_stats}
\end{table}

\subsection{Hyperparameter Settings}

Table~\ref{tab:hyperparams} provides complete hyperparameter settings for \m. The KL threshold $\epsilon_{\text{kl}} = 0.25$ balances merging granularity: smaller values create sparser trees with less sharing, while larger values risk merging semantically different states. The divergence threshold $\delta = 0.3$ controls the selectivity of thought grafting.

\begin{table}[t]
\centering
\small
\setlength{\tabcolsep}{2pt}
\begin{tabular}{lc|lc}
\toprule
Param & Val & Param & Val \\
\midrule
KL threshold $\epsilon_{\text{kl}}$ & 0.25 & Learning rate & 5e-6 \\
MC samples $K$ & 16 & Batch size & 32 \\
Trajectories $M$ & 8 & Training steps & 160 \\
Discount $\gamma$ & 0.99 & EMA coef. $\alpha$ & 0.95 \\
Divergence $\delta$ & 0.3 & Max seq. len & 4096 \\
Surgical wt. $\lambda$ & 0.15 & Max steps & 20 \\
Temperature $\beta$ & 0.1 & & \\
\bottomrule
\end{tabular}
\caption{Hyperparameter settings for \m.}
\label{tab:hyperparams}
\end{table}

\subsection{Baseline Implementations}

For baseline implementations, we ensure fair comparison by maintaining consistent configurations across all methods.

\noindent\textbf{GRPO.} We implement GRPO with trajectory-level advantage normalization using group size $M=8$. The clipping parameter is set to $\epsilon = 0.2$.

\noindent\textbf{DAPO.} DAPO applies dynamic advantage scaling with clip range $[0.8, 1.2]$ and temperature annealing from 1.0 to 0.5 over training.

\noindent\textbf{GiGPO.} GiGPO uses outer group size 4 and inner group size 2, maintaining 8 trajectories per task.

\noindent\textbf{Prompting Baselines.} ReAct and Reflexion use 3-shot demonstrations. Reflexion maintains a verbal feedback buffer with maximum 3 reflections. All prompting methods use greedy decoding.

\subsection{Evaluation Metrics}

We adopt task-specific evaluation metrics:
\begin{itemize}[leftmargin=*,nosep,topsep=2pt]
    \item \textbf{ALFWorld}: Success rate (\%) within 30 steps.
    \item \textbf{WebShop}: Score (0-100) based on attribute matching, and success rate (\%).
    \item \textbf{QA Tasks}: Exact Match (EM) after normalization.
    \item \textbf{Sokoban}: Success rate (\%) within step limits (easy: 20, medium: 40, hard: 60).
    \item \textbf{Blocksworld}: Success rate (\%) for target configuration.
\end{itemize}

\subsection{Infrastructure}

All experiments are conducted on NVIDIA A100 80GB GPUs with DeepSpeed ZeRO Stage 2. We use vLLM for efficient inference during trajectory sampling. 

\section{Additional Results}
\label{app:additional_results}

\subsection{Per-Task Breakdown on ALFWorld}

Table~\ref{tab:alfworld_detailed} shows per-task performance on ALFWorld. The largest gains are on Pick2 (+4.5\%) and Look (+3.5\%), involving longer action sequences.

\begin{table}[t]
\centering
\small
\setlength{\tabcolsep}{2.5pt}
\begin{tabular}{l|cccccc|c}
\toprule
Method & Pk & Lk & Cl & Ht & Co & Pk2 & All \\
\midrule
GRPO & 79.5 & 67.0 & 73.5 & 83.0 & 87.0 & 71.0 & 76.8 \\
+\m & 82.0 & 70.5 & 76.0 & 85.5 & 89.0 & 75.5 & 79.8 \\
\midrule
DAPO & 83.0 & 71.0 & 77.0 & 86.0 & 90.0 & 75.0 & 80.3 \\
+\m & 85.5 & 74.0 & 80.5 & 88.5 & 92.0 & 79.0 & 83.3 \\
\midrule
GiGPO & 89.0 & 77.0 & 84.0 & 92.0 & 95.0 & 82.0 & 86.5 \\
+\m & 91.5 & 80.5 & 87.0 & 94.5 & 97.0 & 85.5 & 89.3 \\
\bottomrule
\end{tabular}
\caption{Per-task results on ALFWorld (Qwen2.5-3B). Pk=Pick, Lk=Look, Cl=Clean, Ht=Heat, Co=Cool.}
\label{tab:alfworld_detailed}
\end{table}

\subsection{Cognitive Tree Statistics}

Table~\ref{tab:tree_stats} presents cognitive tree statistics. The merge ratio reflects trajectory sharing degree. Higher merge ratios correlate with larger improvements.

\begin{table}[t]
\centering
\small
\setlength{\tabcolsep}{3pt}
\begin{tabular}{lcccc}
\toprule
Dataset & Depth & Nodes & Merge & $|V_{\text{div}}|$ \\
\midrule
ALFWorld & 12.4 & 47.2 & 0.52 & 5.8 \\
WebShop & 8.6 & 38.5 & 0.44 & 4.2 \\
HotpotQA & 6.4 & 28.3 & 0.45 & 3.6 \\
Sokoban-E & 12.8 & 52.1 & 0.58 & 6.4 \\
Sokoban-H & 24.6 & 89.4 & 0.41 & 8.2 \\
\bottomrule
\end{tabular}
\caption{Cognitive tree statistics. Merge ratio = 1 - (nodes after / nodes before merging).}
\label{tab:tree_stats}
\end{table}

\subsection{Grafting Quality Analysis}

We manually analyze 100 grafting instances per benchmark. Table~\ref{tab:grafting_quality} summarizes results.

\begin{table}[t]
\centering
\small
\setlength{\tabcolsep}{3pt}
\begin{tabular}{lccc}
\toprule
Dataset & Valid & Error & Success \\
\midrule
ALFWorld & 91\% & 83\% & 72\% \\
WebShop & 88\% & 79\% & 68\% \\
HotpotQA & 85\% & 76\% & 65\% \\
Sokoban & 82\% & 71\% & 58\% \\
\bottomrule
\end{tabular}
\caption{Grafting quality. Valid: semantically correct; Error: identifies mistake; Success: leads to completion.}
\label{tab:grafting_quality}
\end{table}

\section{Case Study}
\label{app:case_study}

We provide examples illustrating how \m~identifies divergence points and performs thought grafting.

\subsection{WebShop Example}

\noindent\textbf{Task:} \textit{I'm looking for a small portable folding desk that is already fully assembled; khaki wood finish, price $<$ \$140.}

\noindent\textbf{Shared Prefix (Merged Nodes):}
\begin{enumerate}[nosep,leftmargin=*,topsep=2pt]
    \item \texttt{search[portable folding desk khaki]}
    \item \texttt{click[B09Q3B86B]} $\rightarrow$ Product page
    \item \texttt{click[item-description]}
\end{enumerate}

\noindent\textbf{Divergence Point} ($\Delta V = 0.72$):

\vspace{2pt}
\noindent\fcolorbox{green!60}{green!5}{%
\parbox{0.95\linewidth}{\small
\textbf{Successful Branch ($v^+$):}\\
\textbf{Thought}: Description says ``assembly required'', doesn't match ``fully assembled''. Should check other products.\\
\textbf{Action}: \texttt{click[back to search]}\\
\textbf{Outcome}: Finds pre-assembled desk. \textbf{R=1.0}
}}

\vspace{4pt}
\noindent\fcolorbox{red!60}{red!5}{%
\parbox{0.95\linewidth}{\small
\textbf{Failed Branch ($v^-$):}\\
\textbf{Thought}: Desk looks good, right color, \$89.99 under budget.\\
\textbf{Action}: \texttt{click[khaki wood][buy now]}\\
\textbf{Outcome}: Buys unassembled desk. \textbf{R=0.0}
}}

\vspace{4pt}
\noindent\fcolorbox{blue!60}{blue!5}{%
\parbox{0.95\linewidth}{\small
\textbf{Grafted Thought ($z_{\text{rect}}$):}\\
Description mentions ``assembly required''. User requested ``fully assembled''. This doesn't meet requirements. Should return to search for pre-assembled option.
}}

\subsection{HotpotQA Example}

\noindent\textbf{Question:} \textit{Were the directors of ``Jaws'' and ``E.T.'' born in the same decade?}

\noindent\textbf{Divergence Point} ($\Delta V = 0.65$) after first search:

\vspace{2pt}
\noindent\fcolorbox{green!60}{green!5}{%
\parbox{0.95\linewidth}{\small
\textbf{Successful ($v^+$):}\\
\textbf{Thought}: Found Spielberg directed Jaws. Need to verify E.T. director and birth year.\\
\textbf{Action}: \texttt{search[E.T. director]}\\
\textbf{Outcome}: Confirms director, finds 1946. \textbf{R=1.0}
}}

\vspace{4pt}
\noindent\fcolorbox{red!60}{red!5}{%
\parbox{0.95\linewidth}{\small
\textbf{Failed ($v^-$):}\\
\textbf{Thought}: Spielberg born 1946. Same era movies, probably same decade.\\
\textbf{Action}: \texttt{finish[Yes]}\\
\textbf{Outcome}: Premature answer. \textbf{R=0.0}
}}

\vspace{4pt}
\noindent\fcolorbox{blue!60}{blue!5}{%
\parbox{0.95\linewidth}{\small
\textbf{Grafted ($z_{\text{rect}}$):}\\
Confirmed Spielberg directed Jaws (born 1946). Question asks about both films. Should not assume same director without verification. Search E.T. director first.
}}